\documentclass[letterpaper,twocolumn,10pt]{article}
\usepackage{usenix-2020-09}

\usepackage{tikz}
\usepackage{amsmath}

\usepackage{filecontents}
\usepackage{amssymb}
\usepackage{amsthm}
\usepackage{mathtools}
\usepackage{soul}
\usepackage{url}  
\usepackage{algorithm}
\usepackage{tcolorbox}
\usepackage[noend]{algpseudocode}
\usepackage{subcaption}
\usepackage{hyperref}
\usepackage{enumitem}
\usepackage{booktabs} 
\usepackage{multirow} 
\usepackage{siunitx} 
\usepackage{amsfonts} 
\usepackage{hyperref}
 
\usepackage[numbers]{natbib}
\usepackage{multirow}
\usepackage{todonotes}

\begin{document}

\pagenumbering{gobble}

\date{}

\title{\vspace*{-0.5in}
{{\normalsize \rm In 31\textsuperscript{st} {\em USENIX Security Symposium}\hrule}}
\vspace*{0.4in}On the Necessity of Auditable Algorithmic Definitions for Machine Unlearning}

\newtheorem{definition}{Definition}
\newtheorem{lemma}{Lemma}
\newtheorem{theorem}{Theorem}
\newtheorem{corollary}{Corollary}

\newcommand{\anvith}[1]{\textcolor{brown}{anvith: #1}}
\newcommand{\varun}[1]{\textcolor{red}{VC: #1}}
\newcommand{\ilia}[1]{\textcolor{green}{ilia: #1}}
\newcommand{\nick}[1]{\textcolor{purple}{HJ: #1}}

\def\@onedot{\ifx\@let@token.\else.\null\fi\xspace}
\def\etal{\emph{et al}\onedot}
\def\ie{\emph{i.e.~}}
\def\eg{\emph{e.g.~}}

\author{
{\rm Anvith Thudi, Hengrui Jia, Ilia Shumailov, Nicolas Papernot}\\
University of Toronto and Vector Institute
} %

\maketitle

\begin{abstract}

Machine unlearning, \ie having a model forget about some of its training data, has become increasingly more important as privacy legislation promotes variants of the right-to-be-forgotten. In the context of deep learning, approaches for machine unlearning are broadly categorized into two classes: exact unlearning methods, where an entity has formally removed the data point's impact on the model by retraining the model from scratch, and approximate unlearning, where an entity approximates the model parameters one would obtain by exact unlearning to save on compute costs. In this paper, we first show that the definition that underlies approximate unlearning, which seeks to prove the approximately unlearned model is close to an exactly retrained model, is incorrect because one can obtain the same model using different datasets. Thus one could unlearn without modifying the model at all. We then turn to exact unlearning approaches and ask how to verify their claims of unlearning. Our results show that even for a given training trajectory one cannot formally prove the absence of certain data points used during training. We thus conclude that unlearning is only well-defined at the algorithmic level, where an entity's only possible auditable claim to unlearning is that they used a particular algorithm designed to allow for external scrutiny during an audit.
\end{abstract}
\section{Introduction}
Machine learning (ML) models are often trained on user data.
While the users may initially allow for their data to be analyzed with ML, they may later retract this authorization. 
This is formalized in the right-to-be-forgotten, established in the European Union's GDPR~\cite{mantelero2013eu}, among other regulatory frameworks. 
Because ML models are known to memorize their training data~\cite{feldman2020does}, this forces model owners to resort to machine unlearning~\cite{cao2015towards} to remove any effect the point had on the trained model.
Machine unlearning is a process by which the model owner proves to the user that a model which had originally learned from the user's data now has been modified to ``forget'' anything it had learned from it. 
While this is an intuitive concept, precisely defining unlearning remains an open problem. 

Indeed, there are several proposals for defining machine unlearning. They can be broadly taxonomized in two categories: \textit{exact} and \textit{approximate} unlearning. 
In exact unlearning, the model is naively retrained from scratch (with the point to be unlearned removed from the dataset). 
This is generally computationally expensive, although it can be made more efficient to a certain extent~\cite{bourtoule2019machine}. 
In approximate unlearning, the model owner instead attempts to save computational resources by directly modifying the model's parameters to unlearn in a small number of updates (ideally, a single update would suffice). This second class of approaches defines machine unlearning in the parameter space~\cite{golatkar2020eternal,golatkar2020forgetting,guo2019certified,graves2020amnesiac}: a model is said to have unlearned successfully when it is close in parameter space to a model which has never learned about the data point.

In this paper, we show that this second line of reasoning on approximate unlearning is erroneous and leads to incorrect definitions of machine unlearning. Our result first stems from our observation that one can obtain arbitrarily similar (which for finite floating point storage means identically the same) model parameters from two non-overlapping datasets. Thus, the ability to achieve a particular point (\ie model) in parameter space is not what defines having learned from a dataset-- instead, it is the trajectories taken by the optimization algorithm on the loss (\ie error) surface during training. Second, \citeauthor{shumailov2021manipulating} recently demonstrated how it is possible to find different samples from the training data that will lead to a similar step of gradient descent~\cite{shumailov2021manipulating}. 
This can be seen by observing that a step of gradient descent differentiates the model's loss averaged over a minibatch (\ie a small number) of training examples. Thus, different minibatches can lead to arbitrarily similar average gradients. Coming back to the gradient descent trajectories, this implies that the value of training from a particular dataset stems from \textit{how} the trajectory was obtained, not the trajectory itself. 

This has a fundamental implication to unlearning: we \textit{cannot} prove unlearning by showing that the parameters of the unlearned model are obtained without training on the unlearned data.
This poor definition of unlearning could be systematically achieved by inspecting a model's training trajectory to identify all steps of gradient descent that involve the point to be unlearned, and then replacing some of the training points in the corresponding minibatches such that the minibatches no longer contain the point to be unlearned but the outcome of the gradient computation remains identical. 

We show this is possible for standard models and datasets in our experiments. Furthermore, we formalize the concept of \textit{forgeability} as the ability to replace data points from one dataset by points from another dataset while keeping the model updates computed on these points identical (up to some pre-defined minimal error). That is, when two datasets are forgeable, the same model parameters can feasibly (upto some small per-step error accounting for noise) be obtained by using either of the datasets. When applied to the particular case where the second dataset is obtained from the first dataset by removing points that need to be unlearned, the two datasets being forgeable implies that unlearning can be achieved at no cost (because learned and unlearned models are indistinguishable). 

This leads us to the conclusion that unlearning can only be defined at the level of the algorithms used for learning and unlearning, and not by reasoning over the model parameters they output. 
This observation is akin to the one made by Dwork et al.~\cite{dwork2006calibrating} for protecting sensitive data when introducing differential privacy: one cannot prevent inference of private information that individuals contribute to a dataset by reasoning at the level of the dataset itself, but rather has to reason about the guarantees provided by algorithms that analyze this dataset. Similarly in our setting for unlearning, we show that one cannot prove unlearning simply by comparing points in the parameter space. 
This invalidates all existing approximate approaches to unlearning:
they 
can only justify unlearning by the process they are enacting, and not the final outcome of their process (which most past works looked at); the only way to know a model unlearned correctly with a given method is to inspect the method itself. 
Indeed, unlearning requires that (1) the procedure of unlearning a training point possesses the properties required to ensure removal of the effect of the point, and (2) the model owner provides auditable evidence that the unlearning procedure was correctly implemented. External scrutiny (\eg taking the form of audits) is required to achieve the second property.
Thus, we propose verifiable computation as a promising direction for achieving unlearning. %

\section{Background}

\subsection{Machine Learning}
\label{sec:background_ml}

Machine learning (ML) is the task of learning a machine (or model) from a dataset, and it falls broadly into two categories: supervised and unsupervised learning. The difference between them is that in the former an entity is given a labeled dataset $D$ whereas in the latter the dataset is unlabelled. We will focus on supervised learning in this paper, in which the dataset consists of points $(\mathbf{x}_i,y_i) \in \mathfrak{X} \times \{1,\cdots,c\}$, where $y_i \in \{1,\cdots,c\}$ is the label of the input $\mathbf{x}_i \in \mathfrak{X}$, and there are $c$ possible labels. The goal of supervised ML is to predict the label $y$ of an unlabelled input $\mathbf{x}$ by knowledge learned from the labeled dataset $D$.

The usual method of learning a supervised model $M:\mathfrak{X} \rightarrow \{1,2,\cdots,c\}$ is to introduce parameters $\mathbf{w}$ which allows one to modify the model in a (piece-wise) smooth way so (ideally) the performance can be optimized smoothly without overshooting. Then one typically defines a loss function $\mathcal{L}: M_{\mathbf{w}}(\mathbf{x}_i) \times y_i \rightarrow [0,\infty)$ which captures the performance of $M_{\mathbf{w}}$ at classifying the dataset $D$ correctly, or in particular how different $M_{\mathbf{w}}(\mathbf{x}_i)$ is from the desired output $y_i$ for every data point $(\mathbf{x}_i,y_i)$ (note throughout this paper we will often use just $\mathbf{x}_i$ as shorthand for the data point $(\mathbf{x}_i,y_i)$). In cases where an analytic solution does not exist, one usually minimizes the loss $\mathcal{L}$ iteratively (and stochastically) by updating the parameters in steps based on individual data points. That is, there is some update rule $g$ which takes in the current parameters $\mathbf{w}_t$ and a data point $(\mathbf{x}_t, y_t)$ and gives new parameters $\mathbf{w}_{t+1} = g(\mathbf{w}_t, (\mathbf{x}_t, y_t))$ which (often) reduces the loss by some amount. This is useful, for instance, to minimize non-convex losses. Once the amount of reduction becomes negligible, we say the model is converged and the training is finished. Note that we use $g$ to denote everything needed to make the update happen, including the optimizer, hyperparameters, etc.

An important class of update functions $g$ which we focus on in this paper are what we will call \textit{mean-samplers}. In this case the dataset $D = \{\mathbf{\hat{x}}_1,\cdots,\mathbf{\hat{x}}_n\}$ consists of minibatches of data $\mathbf{\hat{x}}_1 = \{(\mathbf{x}_1,y_1),\cdots,(\mathbf{x}_b,y_b)\}$, and then the update function is defined as $\mathbf{w}_{t+1} = g(\mathbf{w}_t,\mathbf{\hat{x}}_t) = 1/b \sum_{i=1}^{b} g(\mathbf{w}_t,(\mathbf{x}_i,y_i))$; one could also without loss of generality not distinguish between mini-batches and data points as the former is simply a particular form of a data point, and throughout the paper we will use the same notation for both (though we make a distinction here for clarity). A particular example of a mean-sampler is the minibatch stochastic gradient descent (SGD) update rule, where for each update, we sample a minibatch (\ie a small number) of data points $\mathbf{\hat{x}}_t = \{(\mathbf{x}_1,y_1),\cdots,(\mathbf{x}_b,y_b)\}$, and update the model to minimize the average loss of these points for stability. This variant of SGD is well known for its applications in deep learning. Note that starting from this point, we will use the term "data point" to represent a minibatch of data points for simplicity of the terminology. The update rule can be written as:
\begin{equation}
    \mathbf{w}_{t+1} = g(\mathbf{w}_t,\mathbf{x}_t) = \mathbf{w}_{t} - \eta \frac{1}{b} \sum_{i=1}^{b} \frac{\partial \mathcal{L}(M_{\mathbf{w_t}}(\mathbf{x}_i),y_i)}{\partial \mathbf{w}}|_{\mathbf{w}_t}
\end{equation}

\subsection{Machine Unlearning}

As a dual to ML, unlearning was first introduced by \citeauthor{cao2015towards}~\cite{cao2015towards} as a method to remove the impact of a data point on the model obtained upon completion of training. Since then further extensions to scenarios where efficient analytic solutions could be found were given~\cite{ullah2021machine,ginart2019making}, and an extension to unlearn deep neural networks (DNNs) were proposed~\cite{bourtoule2019machine,guo2019certified,sekhari2021remember,graves2020amnesiac,golatkar2020forgetting,golatkar2020eternal}. With the growing field of unlearning, there has emerged two categories of machine unlearning algorithms~\cite{thudi2021unrolling}: exact and approximate unlearning, differing by how unlearning is done, and also how the concept of "unlearning" is understood.

Exact unlearning for DNNs is based on retraining. In detail, the model owner needs to discard the old model, remove the data points that are required to be unlearned, and train a new model on the modified dataset. Here the concept of unlearning is viewed at the algorithmic level, that is the data to be unlearned is not used at all by the algorithm to obtain the new model. Thus as long as the algorithm is executed correctly, the model owner can claim the data points are unlearned correctly. However, the issue with exact unlearning comes from the computational cost: training DNNs is costly~\cite{brown2020language,devlin2018bert}, and the dataset used may contain thousands or even millions of data points, so retraining the entire DNN for each unlearning request is unaffordable. Therefore, research in exact unlearning focused on proposing an efficient way of retraining. For example, the SISA approach~\cite{bourtoule2019machine} suggests to partition the dataset into a few non-overlapping shards and train a DNN on each of the shards. Then the final model would be an ensemble of these DNNs. In this case, when an unlearning request happens, the model owner only needs to retrain the DNN on one of the shards, resulting in a significant speedup compared to retraining the DNN on the entire dataset. However, even after this speedup, the costs may still be too high for some.

On the other hand, approximate unlearning, as the name indicates, attempts to approximate what the parameters of the model would have looked like if the data points to be unlearned were not in the training dataset from the beginning. These approaches typically directly modify the parameters of the trained model, that is, such approaches understand the concept of unlearning from the parameter level. For example, Amnesiac Machine Learning~\cite{graves2020amnesiac} records the updates in the parameter space made by each data point during training, and later when some of the data points are required to be unlearned, the model owners can simply subtract the corresponding updates from the model's final parameters. It can be easily seen such a method is much more efficient than any retraining-based exact unlearning methods. However, the disadvantage comes from the fact that the unlearning property provided is harder to understand (especially to the end user): the data points to be unlearned  are still involved in some of the decisions made to obtain the final model. Existing works have proposed to use techniques such as membership inference~\cite{shokri2017membership} to verify the effectiveness of approximate unlearning~\cite{graves2020amnesiac,baumhauer2020machine} to show that their approximately unlearned models cannot be easily distinguished from models that are not trained on the data points to be unlearned. Alternatively, others compare the similarity of the approximately unlearned models parameters to exactly unlearned models parameters~\cite{golatkar2020eternal,golatkar2020forgetting,wu2020deltagrad,thudi2021unrolling}. 

However, it is worth noting that approximate unlearning methods are basing their "unlearning" on reproducing properties of exactly unlearned models. Hence their claim is tied to the claim that retraining, which is the foundation of all existing exact unlearning methods, is well-defined.

\subsection{Proof of Learning}
\label{ssec:pol}

The concept of Proof-of-Learning (PoL) was introduced by \citeauthor{jia2021proof}~\cite{jia2021proof}. It is designed to provide evidence that someone has spent the effort to train a machine learning model, in particular DNNs, so that later it can be confirmed by a verifier (\ie a third-party authority such as a regulator). This is done by keeping a log collected during training which facilitates reproduction of the alleged computations an entity took. At the core of it, the Proof-of-Learning (\ie the log) documents intermediate checkpoints of the model, data points used, and any other information required for the updates during training (\eg hyperparameters that define the update rule $g$ at each step). We follow a similar line of reasoning as PoL to introduce the strawman concept of "Proof of Unlearning" in \S~\ref{sec:strawman} and build on this logging framework to show why unlearning is not always well-defined. The PoL can be written in the form of $\{(\mathbf{w}_i,\mathbf{x}_i,g_i)\}_{i \in J}$ for some indexing set $J$, where we use $\mathbf{x}_i$ to represent a minibatch (recall our remark in \S~\ref{sec:background_ml} that a minibatch is just a type of a data point, and hence notationally we will not distinguish between individual data points and minibatches). Alternatively, when the update rule is understood from context, the PoL can be represented by just $\{(\mathbf{w}_i,\mathbf{x}_i)\}_{i\in J}$. 

Reproducing the alleged computation is synonymous to showing its plausibility. When it comes to the stage of verification, the verifier checks a PoL's validity (\ie plausibility) by reproducing the $t^{th}$ intermediate checkpoint of the model based on the information given in the log, including the $(t-1)^{th}$ checkpoint, data points used at this step, and the same $g$ as defined in the log. Then the verifier computes the distance between the $t^{th}$ checkpoint in the log and the reproduced $t^{th}$ checkpoint in the parameter space. This distance is called the verification error, and we say this update is valid if the verification error is below a certain threshold. Here the threshold is introduced based on the observed noise/randomness seen when reproducing computations on different hardware (which comes from the randomness in back-end floating point computations). Indeed, this should be the only source of verification error since all other possible randomness is ``seeded'' by the information provided in PoL: for instance, the effect of sampling minibatches is fixed by knowing which minibatch the step was computed on. %
It is worth noting that computing a distance between model parameters is essentially a mapping from high-dimensional space to a scalar, so it is possible that there exists multiple checkpoints resulting in a verification error below the threshold. In other words, PoL does not guarantee that the computation of a step results exactly in the next checkpoint, it just verifies that the next checkpoint is plausible.

\subsection{Data Ordering Attack}
\label{ssec:data-ordering}

While the existence of PoL suggests that models could be uniquely mapped to their training points, our work introduces the concept of \textit{forging} to show this is not always the case.
Part of our motivation for forging comes from a recent discovery by \citeauthor{shumailov2021manipulating}~\cite{shumailov2021manipulating} who showed that stochastic gradient descent relies heavily on its underlying randomness, and ordinal effects coming from data can be significant. This shows that assumptions made to obtain analytical generalization guarantees for SGD are indeed important in practice. More precisely, the authors found that the order in which data points are used for training can be sufficient for an adversary to completely change the training dynamics -- all the way from stopping and slowing learning, to learning behaviours that are not present in the training dataset (\eg as is done in poisoning, but without actually using poisoned data). This is because one can force a step of SGD to approximately compute a given gradient by carefully selecting points the step is computed on. In doing so, they derived conditions in which ordinal effects can have a significant impact on learning in the case of SGD. 

The concept of forging introduced in this paper utilises the theoretical findings of~\citeauthor{shumailov2021manipulating} to search for alternative minibatches that produce that same or a similar parameter update. Note that  originally~\citeauthor{shumailov2021manipulating} constructed minibatches from data contained in the same dataset, whereas here we assume availability of a different dataset. However, our algorithm for finding forging examples (which our analytical results show exist) analogously uses the idea of random sampling and picking the sample that best produces the update desired.

\section{Revisiting Machine Unlearning}
\label{sec:motivation}

\subsection{When is Unlearning Achieved?}
\label{ssec:approximate-unlearning-not-defined}

\paragraph{Threat Scenario:} In this paper we consider the following threat model. The adversary is the model trainer who is trying to prove to a verifier that they unlearned, \ie (re-)trained not using certain data points, without changing the original model. Our adversary has access to all intermediate model checkpoints, as well as, the data points and hyper-parameters used at each step. This adversary is white-box and has access to any information one gains by training the models themselves.

\paragraph{Issues with Defining Unlearning:}
In previous literature related to unlearning~\cite{thudi2021unrolling,graves2020amnesiac,golatkar2020eternal,golatkar2020forgetting}, a model is said to have unlearned if the impact of the data points to be unlearned is removed from the model. Here, removing the impact of the data points from the model is equivalent to obtaining a model that one could get from retraining without those points. In other words, a model is unlearned if its parameters are in the subspace of the parameter space that can be achieved by training on a dataset that does not contain the points to be unlearned. Furthermore, other work also require the process of obtaining those parameters to reproduce the same distribution one gets from retraining \cite{golatkar2020eternal,golatkar2020forgetting,guo2019certified}. Although such a definition sounds intuitive, in this paper we question its parameter-space assumption; \ie are there parameters such that the model is both unlearned and not unlearned at the same time? In other words, we argue that existence of parameters reachable with and without given data entails that unlearning is not well-defined over the parameter space.

For an analogy of what unlearning currently is, think of training a neural network as building a castle with Lego blocks, where the blocks are borrowed from many of your friends (\ie represent your data). Assume, as is the case in deep learning, that the training algorithm is such that you go to each of your friends and request a block from them; that is, you expect that each friend will contribute a block every few centimeters of the castle built (\ie the data points to be unlearned are used in every epoch). When one of your friends wants their blocks back, the naive way is to disassemble the entire castle and take the blocks out, \ie retraining. Existing works either (1) suggest to use one friend's data only in a single wall and ensemble the walls to build the castle~\cite{bourtoule2019machine} (this way when one friend asks for the blocks, only one wall needs to be dissembled and re-building is faster); or (2) introduce special tools to take the blocks out without destroying the castle (\ie approximate unlearning \eg amnesiac machine learning~\cite{graves2020amnesiac}).

However, imagine a case where two of your friends have lent identical blocks to you. Then when one of them asks for their blocks back, you have an identical one from the other friend which was not yet used in the castle -- does it make sense for you to not touch your well-built castle and give the blocks of the latter friend to the former? That is, if the model owner (or deployer) has 2 identical data points and only one of them is required to be unlearned, is it justifiable for them to not do anything about the model? If this sounds acceptable, imagine another case that you are great at drawing. You paint the blocks that your friends ask for, so that they look exactly the same as blocks from some other friends, and pretend you did not use the blocks at all when a friend asks for their block back. Such a `painting' technique would make the parameter-space definition of unlearning described earlier trivial to achieve: the Lego castle could have been obtained without access to your friend's brick because there exists another friend's brick that is identical. 

This painting technique is in some sense metaphorical of the concept of \textit{forging} we present in our work:  we show that when a data point needs to be unlearned, there can exist a dataset not containing that data point which allows us to obtain the same model parameters as the model we obtained before without having to unlearn. In other words, the subspace of the parameter space that can be achieved by training on a dataset that does not contain the points to be unlearned can at times be equivalent to the parameter space one gets by training with those points. The concept of forging which leads to this claim will be introduced formally in \S~\ref{sec:forging}. In this case, parameter-space definitions of unlearning are essentially saying no unlearning needs to be done at all. Such self-contradiction indicates that unlearning cannot be defined at the level of the parameters. It, in particular, invalidates approximate approaches to unlearning.

\subsection{A Strawman for an Auditable Approach to Unlearning: Proof-of-Unlearning}
\label{sec:strawman}

In the section above we discuss why defining unlearning at the level of the parameters, as is done in approximate unlearning, does not make sense in the presence of forging. This naturally leads us to consider exact unlearning methods that define unlearning at the level of the algorithm. In this section we show that forging can also invalidate exact unlearning approaches. To this end, we first describe an attempt to create proofs that show retraining-based unlearning is done correctly, and then explain why such proofs are still vulnerable to the concept of forging we introduce in this paper. We conclude that while unlearning is best defined at the level of the learning and unlearning algorithms, it also requires external scrutiny to provide meaningful guarantees. The strawman we present here is not sufficient for auditable claims of unlearning. We discuss possible directions for enabling this scrutiny in \S~\ref{ssec:verifiable_ML}.

Since exact unlearning is retraining-based, one may think that verifying the correctness of unlearning is equivalent to verifying that an entity retrains the model on a dataset without the data points to be unlearned. Recall that, as described in \S~\ref{ssec:pol}, PoL implicitly assumes that the correctness of training is equivalent to the plausibility of the log returned by the model owner, which contains model checkpoints and other information obtained during training. Thus it is natural to make use of the proof-of-learning (PoL) here to show that one obtained the unlearned model with a sequence of model updates that do not involve the point to be unlearned in the minibatches sampled from the dataset. If a PoL not involving the data points to be unlearned is submitted to the verifier and passes verification, this would attest that exact unlearning has been done. One might coin the term Prof-of-\textit{Un}learning for such an approach, and shorten it to PoUL.

It turns out there is a flaw in this reasoning because the threat models for PoL (training) and PoUL (unlearning) are different. For PoL (training), the authors can claim correctness of training is closely related to the plausibility of the training log because it is observed that the error accumulates during training so that the same initial parameters can lead to very different final parameters--making the log unique. However, the error at one training update is much smaller, so it may be possible for an adversary to fake one step of training, but the only way to obtain an entire plausible log is to do the training correctly. One key reason for this is that the adversary only has access to the final model. Yet, in the setting of PoUL (unlearning), \textit{the potential adversary has access to the entire log of the model's training run before unlearning} because the adversary is the model owner. Therefore, if an adversarial model owner wants to pretend they unlearned, they may only need to create a few plausible steps instead of an entire plausible log. Using the same analogy as before, PoUL can be thought of as an album containing photos of the Lego castle taken as one builds the castle. A model-stealing adversary only steals the built castle so it is hard to make such an album without actually building the castle. Instead, when the model owner attempts to avoid unlearning, they already have the album and only need to "paint" some of the photos in the album to make it look like that some Lego blocks were never used during building.

To formalize this intuition, given the PoL log $\{\mathbf{w}_i,\mathbf{x}_i\}$, the adversary has some map $B: \{\mathbf{w}_i,\mathbf{x}_i\} \rightarrow \{\mathbf{w}_i,\Tilde{\mathbf{x}}_i\}$ over the PoL logs which swaps a particular data point $\mathbf{x}^*$ with some other data point $\Tilde{\mathbf{x}}^*$ every time it occurs, and this new data point still produces the same update (\ie the log returned by $B$ still passes verification under some threshold $\epsilon$). One way to obtain such a map is through an analog to the data ordering attack of \citeauthor{shumailov2021manipulating}~\cite{shumailov2021manipulating} (recall \S~\ref{ssec:data-ordering}), as we will see in \S~\ref{sec:attack}. Thus without changing any of the checkpoints the adversary still has a valid PoL which does not contain the point to be unlearned. This PoL is therefore a valid PoUL, despite having been obtained without exact unlearning (\ie without retraining). Furthermore, what if this map exists for any possible PoL that could be obtained by training on the original dataset? Then this Proof-of-Unlearning is not sufficient to determine if the model owner did in fact retrain without the data point, or just applied the mapping at the steps involving this point. The takeaway is that there is no sequence of checkpoints that would be sufficient alone to determine if an entity unlearned correctly; phrased another way, unlearning is not well-defined by just the final model (seen in \S~\ref{ssec:approximate-unlearning-not-defined}), or the sequence of models that lead to it.

The map $B$ we mention above is what we call the \textit{forging map} and explain more formally in \S~\ref{sec:forging}. We also show why such a map commonly exists for frequently used ML algorithms and datasets in \S~\ref{ssec:Mean_Sampling}. One can see now that with forging, the parameter-level definition of unlearning is self-contradicted, while the proposed way to verify unlearning defined from the level of algorithms can be bypassed by an adversary. This calls for an algorithmic perspective on unlearning that supports external scrutiny, as we discuss in \S~\ref{ssec:verifiable_ML}. %

\section{Introducing Forging}
\label{sec:forging}

Here we introduce the \textit{forging map}
$B$ formally. We describe the consequences of such a map, in particular, our ability to characterize equivalence classes between datasets. Beyond the intrinsic motivation for unlearning (we explain how our formalism helps us define unlearning in \S~\ref{ssec:conseq_to_defining_unl}, and demonstrate when it can be ill-defined), we further motivate in our discussion later why it may be beneficial to study forging maps in other ML contexts (see \S~\ref{ssec:extending_framework}).

\subsection{Defining Forgeability and Forging Map B}
\label{ssec:definition}

\paragraph{Defining Feasibility and PoL}
We first remark on how we define a feasible model from a dataset, which is fundamental in later asking if a model could come without training on a particular data point. In what follows, we treat feasibility in greater generality. A model coming from a fixed dataset $D$ should have a sequence of data points in D and checkpoints that lead to the final weights when following a given update rule, \ie a PoL log~\cite{jia2021proof}. Furthermore if given such a sequence, then we say the final model could come from the dataset (note we do not assume any restrictions on data ordering which future work could add). \textit{Hence the log produced by proof-of-learning is a necessary and sufficient condition to define feasibility}, as models that do come from the dataset have such a log (necessary), and any model that has such a log can come from the dataset (sufficient).

Nevertheless, because of noise in back-end computations, one cannot expect to exactly reproduce a given training sequence from a PoL log~\cite{jia2021proof}. Hence one must allow some $\epsilon$ per-step error to make the definition practical; exactly measuring how large this $\epsilon$ should be is an active research question, and is likely to be case dependent. As such in what follows we will introduce an error parameter, but will not comment on what is necessary or sufficient in practice. This issue is partly bypassed by proving results for all $\epsilon > 0$ (\ie Theorem \ref{thm:prob_forging}). However, instantiations of forging attacks (\S~\ref{sec:attack}) which bypass a particular PoL verification scheme may not pass another PoL verification scheme, hence practically an entity might be happy with a verification scheme that defends against all known attacks to unlearning verification.

As a last remark, note that although we use PoL logs to define feasibility, PoL's original threat model is model stealing. As such, we are not proposing an attack against PoL, and an attack against it is not relevant to us.

\paragraph{Defining Valid Logs.} 
As described in \S~\ref{sec:background_ml}, we denote by $g(\mathbf{w},\mathbf{x}) = \mathbf{w}'$  updating the model parameters $\mathbf{w}$ to $\mathbf{w}'$ where $g$ is a training algorithm/update rule.
We then define a \textbf{\textit{valid $(g,d,\epsilon)$ log}} as a sequence of $\{(\mathbf{w}_i,\mathbf{x}_i)\}_{i \in J}$ for some countable indexing set $J$, such that $\forall i \in J$ $d(\mathbf{w}_{i+1}, g(\mathbf{w}_i,\mathbf{x}_i)) \leq \epsilon$ for some training function $g$ and metric  $d$ on the parameter space. The threshold $\epsilon$ is a tolerance parameter allowing the verifier to account for some numerical imprecision when reproducing the update rule. In other words, a valid sequence is a PoL sequence such that each of its intermediate checkpoints can be reproduced by the verifier retraining for the corresponding step using knowledge of the previous checkpoint (\eg hyperparameters) provided in the PoL.

\paragraph{Defining Sets of Logs.} 
For fixed initial parameters $\mathbf{w}_0$, 
fixed hyperparameters, and dataset $D$, we define \textbf{\textit{Dataset D's Logs}} $H_{D_{g,d,\epsilon}}$ as the set of all valid $(g,d,\epsilon)$ logs stemming from $\mathbf{w}_0$ (we will drop mentioning the $\mathbf{w}_0$ as what it refers to will be explicit from context). We further drop the ${g,d,\epsilon}$ for notational simplicity, though the reader should be aware of the implicit assumption of these which will be clear in context.

\paragraph{Defining the Forging Map and Forgeability.}
Now we define what the forging map $B$ is. For two datasets $D$ and $D'$, and logs stemming from the same $\mathbf{w}_0$, a \textbf{\textit{forging map}} $B: H_{D g,d,0} \rightarrow H_{D' g,d,\epsilon}$ is a map such that $B(\{(\mathbf{w}_i,\mathbf{x}_i)\}_{i \in J}) = \{(\mathbf{w}_i,\Tilde{\mathbf{x}}_i)\}_{i \in J} \in H_{D' g,d,\epsilon}$, \ie is an identity over the parameter space but maps data points $\mathbf{x}_i \in D$ to $\Tilde{\mathbf{x}}_i \in D'$ such that $\Tilde{\mathbf{x}}_i$ satisfy a valid $\epsilon$ log condition. That is to say $d(\mathbf{w}_{i+1},g(\mathbf{w_i},\Tilde{\mathbf{x}}_i)) \leq \epsilon$ $\forall i$. When such a forging map exists, we say $D'$ \textbf{\textit{forges}} $D$ (with $\epsilon)$. Intuitively, this means that we can use the forging map $B$ to replace points from a PoL constructed on $D$ by points from $D'$ and still obtain a valid PoL.

When there also exists an accompanying forging map in the other direction $\Tilde{B}: H_{D' g,d,0} \rightarrow H_{D g,d,\epsilon}$, we say that $D$ and $D'$ are \textbf{\textit{forgeable}} (with $\epsilon$).

Though one could define similar conditions for $H_D$ and $H_D'$ such that the two stem from different initial parameters and different $(g,d,\epsilon)$, we use the same $w_0,g,d,\epsilon$ in our definitions because that is sufficient to prove our results on unlearning. Using different $w_0,g,d,\epsilon$ may be interesting in the context of research on forging for other tasks in ML, and we leave its investigation as future work. Furthermore, our definition of the forging map is a map from valid zero-error logs into logs with some error $\epsilon$; notably, this means we are unable to compose forging maps because the input and output thresholds do not match unless $\epsilon = 0$. To alleviate this lack of transitivity, one might consider a variation of our forging map $B$, $B_{\epsilon}: H_{D g,d,\epsilon} \rightarrow H_{D' g,d,\epsilon}$ and thus one can compose $B_{\epsilon}: H_{D g,d,\epsilon} \rightarrow H_{D' g,d,\epsilon}$ and $B'_{\epsilon}:H_{D' g,d,\epsilon} \rightarrow H_{D'' g,d,\epsilon}$ as their output and input spaces match, respectively. %
We mention this for completeness, for the reader who might be interested in what extensions could be made to our framework to allow different forms of analysis. Notably, we will not consider this extension in our work as it is unnecessary to prove our results and it introduces various technical complications, such as $H_{D g,d,\epsilon}$ never being finite (have sequences of open balls in the parameter space which is a real vector space), where being finite is useful when we introduce probabilistic forging in~\S~\ref{ssec:Mean_Sampling}.

\subsection{Forgeable and Models} The above indistinguishability presented by forgeability becomes even stronger when both $D$ and $D'$ are forgeable with $\epsilon = 0$. We will use $H_D(w)$ to denote all parameters obtained by $g,d,0$ logs starting from $\mathbf{w}_0$ with $D$ and similarly $H_{D'}$.

\begin{lemma}
\label{lemma:fwd_dir_characterizing}
If $D$ and $D'$ are forgeable with $\epsilon = 0$, then $H_D(w) = H_{D'}(w)$.
\end{lemma}

\begin{proof}
Note then we have that, letting $H_D(w)$ denote all parameters obtained by $g,d,0$ logs starting from $\mathbf{w}_0$ with $D$ and similarly $H_{D'}$, that $H_D(w) \subseteq H_{D'}(w)$ by the definition of the forging map $B: H_D \rightarrow H_{D'}$ (\ie identity over parameters) and similarly by the forging map in the other direction, $H_{D'}(w) \subseteq H_{D}(w)$. Thus we have $H_D(w) = H_{D'}(w)$.
\end{proof}

In words, what we have is that \textit{the set of all models achieved by training on $D$ using $g$ is the same as the set of all models achieved by training on $D'$ using $g$}. In fact, in this case (when $\epsilon = 0$) we have that forgeability defines an equivalence relation.

\begin{lemma}
$\epsilon = 0$ Forgeability is an equivalence relation
\end{lemma}

\begin{proof}
We check the three defining properties of an equivalence relation. Clearly $D$ is forgeable with itself (by the identity map), so forgeability satisfies reflexivity. Furthermore $D$ is forgeable by $D'$ means $D'$ is forgeable by $D$ (as by definition the map has an inverse), so forgeability satisfies symmetery. Lastly if $D$ is forgeable with $D'$ by forging map $B_1$, and $D'$ is forgeable by $D''$ with forging map $B_2$, and both are $\epsilon = 0$, then $B_2 \circ B_1$ is a bijective map with error $\epsilon = 0$, thus have $D$ and $D''$ are forgeable with $\epsilon = 0$; so forgeable with $\epsilon = 0$ satisfies the transitivity condition.
\end{proof}

Often when dealing with a set, it is preferable to have another characterization of that set, \ie another defining property for that set which might be easier to work with and understand. We can do that here by extending Lemma~\ref{lemma:fwd_dir_characterizing} to completely characterize the $\epsilon = 0$ equivalence class as datasets that produce the same sequences of parameters:

\begin{theorem}[Characterizing $\epsilon = 0$ Equivalence Classes]
\label{thm:charaterizing_equiv_class}
Two datasets $D$ and $D'$ are $\epsilon = 0$ forgeable iff $H_D(W) = H_{D'}(W)$
\end{theorem}

\begin{proof}
We already proved the forward direction in Lemma~\ref{lemma:fwd_dir_characterizing}. We prove the reverse direction as follows. Let $D_{w_n}$ be the set of all sequences of data points in $D$ that produce the sequence of parameters $\{w_n\}_{n \in J}$ and $D'_{w_n}$ the set of all sequences of data points in $D'$ that produce $\{w_n\}_{n \in J}$ (which we can find as the set of parameters produced by the datasets is the same). Then by the axiom of choice, we can choose a particular sequence of data points $\{x_n'\}_{n \in J} \in D'$ and map $D_{w_n}$ to $\{x_n'\}_{n \in J}$ and note this satisfies $\epsilon = 0$ error as the parameters are the same. We can do this for all sequences of parameters $\{w_n\}_{n \in J}$ and thus have our forging map $B$ from $H_D$ to $H_{D'}$. Note we did not put any constraint on which was $D$ and $D'$ in this argument, so we completely analogously obtain the forging map $\Tilde{B}$ from $H_{D'}$ to $H_{D}$. Thus we have $D$ and $D'$ are $\epsilon = 0$ forgeable, proving the reverse direction.
\end{proof}

Although this theorem adds little in the context of our work on defining unlearning, it is a very notable property of our forging framework that supports the idea that if two datasets are completely forgeable, that translates to them producing all the same models. We later discuss how future work might use this observation in \S~\ref{ssec:extending_framework} and \S~\ref{ssec:verifiable_ML}; in particular, this has some possible important consequences to understanding what makes datasets produce good models and which do not.

\subsection{Consequence to Defining Unlearning.}
\label{ssec:conseq_to_defining_unl}

Recall that in \S~\ref{sec:motivation} we drew an analogy from unlearning data points used in an ML model to taking out blocks from a Lego castle. We claimed that if one is good at ``painting'' then (1) the definition of approximate unlearning self-contradicts, and (2) a proof cannot be created for exact unlearning. Now that ``painting'' is formally defined as forging, we will formally explain what it means for unlearning.

For now, we assume the scenario where there exists a forging map from $D$ to $D'$ where $D'$ does not contain a particular point $\mathbf{x}^*$ (we later show in \S~\ref{sec:methods} such a scenario is common), and our model is first trained on $D$. Since we have done the training, we have access to the PoL logs $H_{D}$. Then what we need to do is to map them into $H_{D'}$ logs which then implies the steps that once used $\mathbf{x}^*$ as a log in $H_D$ no longer do. To be specific, for all $k$ such that the $k^{th}$ step $\mathbf{w}_{k+1}=g(\mathbf{w}_{k}, \mathbf{x}_{i})$ uses a point required to be unlearned (\ie $\mathbf{x}^* \cap \mathbf{x}_i \neq \emptyset$, where it is understood the latter could be a minibatch), after forging $B(\{(\mathbf{w}_i,\mathbf{x}_i)\}_{i \in J}) = \{(\mathbf{w}_i,\Tilde{\mathbf{x}}_i)\}_{i \in J}$ we have that step no longer uses $\mathbf{x}^*$. This is as, by the definition of a forging map in \S~\ref{ssec:definition}, $\{(\mathbf{w}_i,\Tilde{\mathbf{x}}_i)\}_{i \in J} \in H_{D'}$ where $D'$ does not contain $\mathbf{x}^*$. In other words, we can use the forging map $B$ to find some data points in $D'$ to mock the updates made by the data points to be unlearned. This results in a valid (\ie with reproduction error below $\epsilon$) $H_{D'}$, where the final model parameters are the same as before unlearning, but the data points to be unlearned are not used in the log of training. 

Such model parameters satisfy what approximate unlearning defines--they are the same as model parameters obtained by not using the unlearned data points at all. However, no modifications were done to the parameters at all. This means that if the model owner wants to approximately unlearn their model, then by definition they do not need to do anything and can claim the unlearning is done. Thus we come to the conclusion that if there exists a forging map between $D$ and $D/\mathbf{x}^*$, then the approximate unlearning is ill-defined.

Now we move on to exact unlearning. We introduced in \S~\ref{sec:strawman} a way to verify the correctness of exact unlearning through an idea that is similar to PoL, and named it Proof-of-Unlearning (PoUL). Let us now outline how a PoUL is challenged by forging. As described in the previous paragraph, the model owner has trained the model before unlearning on $D$ and thus obtained $H_D$. Then with the existence of a forging map $B$, they can easily create a new log $H_{D'}$ to claim that the exact unlearning is done. Such creation is possible because (1) in theory PoL only verifies if the computation from a checkpoint $\mathbf{w}_{i}$ to the next (\ie $\mathbf{w}_{i+1}$) is plausible, which is essentially what a forging map between $D$ and $D/\mathbf{x}^*$ does: find another set of data points so that the update is plausible; (2) PoL was originally designed for ownership so it was designed for a threat model where the adversary does not have access to the log. This assumption is not valid in the case of unlearning since the adversary  aiming to circumvent unlearning is the model owner themselves. Thus, forging also invalidates PoUL.

The existence of a forging map $B$ from $H_{D}$ to $H_{D'}$ implies that we can manipulate how data points are sampled such that the distributions of checkpoints in the parameter spaces corresponding to $H_{D}$ and $H_{D'}$ respectively are identical. Specifically, for a log $\{(\mathbf{w}_i,\mathbf{x}_i)\}_{i \in J}$, we have $B(\{(\mathbf{w}_i,\mathbf{x}_i)\}_{i \in J}) = \{(\mathbf{w}_i,\Tilde{\mathbf{x}}_i)\}_{i \in J} \in H_{D'}$. We obtain the preimage of a specific sequence of data points $\{\Tilde{\mathbf{x}}_i\}_{i \in J}$ by $preimage(\{\Tilde{\mathbf{x}}_i\}_{i\in J}) = B^{-1}(\{\mathbf{x}_i\}_{i \in J})$ (where we drop the $\{\mathbf{w}_i\}_{i \in J}$ variable as $B$ is an identity on them and we are allowing $\epsilon$ error in the parameters). That is, preimages of $\{\Tilde{\mathbf{x}}_i\}_{i \in J}$ are the sequence of data points from $D$ which can be substituted with the sequence $\{\Tilde{\mathbf{x}}_i\}_{i \in J}$ by applying the forging map. We then constrain the probability of sampling the specific sequence $\{\Tilde{\mathbf{x}}_i\}_{i \in J}$ as the sum of the probabilities of the specific data sequences in $preimage(\{\Tilde{\mathbf{x}}_i\}_{i \in J})$, \ie $\sum_{\{\mathbf{x}_i\}_{i \in J} \in preimage(\{\Tilde{\mathbf{x}}_i\}_{i \in J})} \mathbb{P}(\{\mathbf{x}_i\}_{i \in J} \in  H_D)$. In other words, if multiple sequences in $D$ can be substituted by one single sequence from $D'$, then the latter should be sampled more freqeuntly in creating $H_{D'}$. By doing so, we can create $H_{D'}$ such that each checkpoint log is from the same checkpoint log distribution as $H_D$.

Now we can easily see that unlearning can not be defined on the parameter space since the existence of a forging map implies that the effect of a data point can be obtained by considering another one (or few) data point(s). Notably, this means one cannot verify unlearning (exact or approximate) at the model-level as there is no universally verifiable unlearning property of the weights. Therefore, to define and verify unlearning, we need to focus on the process/algorithm the unlearned model was obtained with.

Particularly, while certified unlearning approaches \cite{sekhari2021remember}\cite{guo2019certified} have the advantage of providing rigorous guarantees at the model level, our work shows that they need to be combined with a verification at the algorithmic level, e.g. verifying a training algorithm is DP, to ensure that the algorithm used was indeed the certified algorithm. For exact unlearning, while simple ideas like proof-of-unlearning do not work to verify unlearning, we later discuss possible ways to verify the process of unlearning in \S~\ref{ssec:verifiable_ML}. They involve facilitating external scrutiny of the algorithm to audit the guarantees of unlearning provided.

\section{Methods for Forging}
\label{sec:methods}

We covered the consequences of the existence of forging maps for unlearning. Now, we proceed to illustrate two theoretical ideas for when two datasets $D$ and $D'$ may be forgable, or when $D$ may be forgeable with $D'$. The first corresponds to the intuition that similar datasets (\ie datasets that contain very similar data points) would allow for an easy exchange of data points. The other is that one dataset that has enough variety in its possible updates at any given parameter would be able to reproduce a large set of other datasets.

Lastly in \S~\ref{ssec:Mean_Sampling} we prove the existence of forging for the mean sampler update function $g$. This function, described in \S~\ref{sec:background_ml}, is important because it is at the core of minibatch SGD. As a consequence, we prove the existence of forging examples  for minibatch SGD (Theorem~\ref{thm:forging_exists}), which are cases when unlearning is not well-defined. Our proof operates under realistic assumptions about how the datasets are obtained. 
Note that this third proof does not require the ideas presented in \S~\ref{ssec:similar} and~\ref{ssec:dense}, but we present them first  to simply illustrate different approaches to looking at forging and characterizing datasets that can be forged (\ie \textit{forging examples}).

\subsection{Forgeability for Similar Datasets}
\label{ssec:similar}

The idea is simply that if one fixes the parameters $\mathbf{w}$ such that the update rule $g$ in only a function of the data point, and furthermore $g_{\mathbf{w}}$ is Lipschitz with respect to the data point $\mathbf{x}$ (under some metric $d_x$ over the input space, and metric $d$ over the parameter space used for verification) for all parameters $\mathbf{w} \in W$. Then, if $L$ denotes the Lipschitz constant, we have if $d(\mathbf{x},\Tilde{\mathbf{x}}) \leq \frac{\epsilon}{L}$ then we could interchange $(\mathbf{w},\mathbf{x})$ with $(\mathbf{w},\Tilde{\mathbf{x}})$ and still have a valid $(g,d,\epsilon)$ log.

\begin{lemma}[Similar Forgiability]
If for all $\mathbf{w} \in H_{D}(w)$, where $H_D$ are all valid $(g,d,0)$ logs starting from $\mathbf{w}_0$, $g$ is Lipschitz with respect to its data point space $X$ (under some metric over the data point space) with metric $d$ over the parameter space in balls $Ball_{\epsilon}(\mathbf{w})$, then there exists a dataset $D'$ such that $D$ and $D'$ are $(g,d,\epsilon)$ forgeable.
\end{lemma}

\begin{proof}
By above remark we simply define for $D = \{\mathbf{x}_1,\cdots,\mathbf{x}_n\}$ a corresponding dataset $D' = \{\Tilde{\mathbf{x}}_1,\cdots,\Tilde{\mathbf{x}}_n\}$ where $\Tilde{\mathbf{x}}_i \in Ball_{\epsilon/L}(\mathbf{x}_i)$, and then define the forging map $B$ sending simply $\mathbf{x}_i \rightarrow \Tilde{\mathbf{x}}_i$ in the logs, and $\Tilde{B}$ vice-versa.
\end{proof}

Note that the data point input space $X$ can be taken to be the smallest compact set containing all the $\mathbf{x}_i \in D$. This is useful for smooth  $g$ functions, as then if also $H_{D}(w)$ are all found in a compact set, then $g$ is guaranteed to be Lipschitz (by having a maximum derivative); we mention this as for DNNs $g$ is smooth (or piecewise smooth depending on the activation functions). This lemma basically shows that one's intuitive notion of similar data points translates to datasets being forgeable for smaller errors.

\subsection{Forgeability of Densely Packed Updates}
\label{ssec:dense}

Another approach is to look at the updates produced by a dataset $D$, and simply find another dataset $D'$ that can produce similar updates for all the parameters obtained in the logs of $D$, i.e., $H_D(W)$. To formalize this idea we have the following:

\begin{lemma}
If for all $\mathbf{w} \in W$ and $\mathbf{x} \in X$, there exists $\mathbf{x}_i \in D'$ such that $d(g(\mathbf{w},\mathbf{x}),g(\mathbf{w},\mathbf{x}_i)) \leq \epsilon$, then for all datasets $D \subset X$ with $H_{D g,d,0}(\mathbf{w}) \in W$, $D'$ forges $D$.
\end{lemma}

\begin{proof}
By the conditions of the lemma we know for any individual step in the log produced by $D$ that there is an $\mathbf{x}_i \in D'$ which would forge it, hence we know there exists a forging map $B: H_D \rightarrow H_{D'}$
\end{proof}

In other words, if dataset $D'$ is sufficiently large and contains data points that lead to sufficiently different updates such that the covering assumption is true, then $D'$ can forge a whole class of datasets, and in a way could be dubbed a ``universal'' forger. Realizing this dataset $D'$ could prove difficult and would likely require specifying the form of $g$. %

\subsection{Forgeability for Mean Sampling}
\label{ssec:Mean_Sampling}

In the particular case of SGD, and more generally mean sampling, we can show the existence of forgeable datasets in a more concrete setting. 
Our approach introduces %
the notion of probability for the existence of a forging map. 

Consider that $\mathfrak{D}$ is some distribution, and that datasets $D$ and $D'$ are sampled from it i.i.d. In particular we assume $D$ and $D'$ consist of minibatches, and each of the elements of the different minibatches are sampled i.i.d from $\mathfrak{D}$. This is the case when training updates are in fact averaging updates (\ie the update rule is a mean-sampler), in the sense that the datasets $D$ and $D'$ are composed of minibatches $\mathbf{\hat{x}} = \{\mathbf{x}_1,\cdots,\mathbf{x}_b\}$, and the update rules take the form of $g(\mathbf{w},\mathbf{\hat{x}}) = 1/b \sum_{i=1}^{b} g(\mathbf{w},\mathbf{x}_i)$ (\ie we also assume we are working with a mean-sampler update rule in what follows). As mentioned earlier in \S~\ref{sec:background_ml}, the commonly used minibatch SGD update rule is an example of a mean-sampler.

Consider the respective logs $H_D$ and $H_{D'}$ for $D$ and $D'$, and in particular a given tuple in a log of $H_D$ $(\mathbf{w},\mathbf{x})$, then note $g(\mathbf{w},\mathbf{x})$ can be viewed as a random variable with mean $\mathbf{\mu}$ and trace of the covariance matrix $\sigma^2 = \sum_{i=1}^{N} \sigma_i^2$ where $\sigma_i^2$ is variance of the ith component of $g(\mathbf{w},\mathbf{x})$ (where the parameter space is $N$-dimensional). This is because $\mathbf{x}$ is a random variable sampled from $\mathfrak{D}$. Similarly for $\mathbf{x'} \in D'$, $g(\mathbf{w},\mathbf{x'})$ is the same random variable as $\mathbf{x'} \sim \mathfrak{D}$.

Then note the mean of $g(\mathbf{w},\mathbf{\hat{x}}) = 1/b \sum_{i=1}^{b} g(\mathbf{w},\mathbf{x}_i)$ is still $\mu$ but now, as we are also mean sampling over all the components which are 1-dimensional (\ie their output is the ith component of the updated parameters)  and thus their individual variances get a $1/b$ factor (by the fact mean sampling 1-dimensional random variables drawn i.i.d introduces a $1/b$ factor to the variance), the trace of the covariance matrix is now $1/b \sum_{i=1}^{N} \sigma_i^2 = (1/b)\sigma^2$.  %

Now let us say that $D'$ contains $n$ minibatches and that the logs of $H_D$ all have lengths less than $m$ and there is a finite number $\alpha$ of logs. Having finite datasets is a realistic assumption (as opposed to infinite datasets which would require infinite storage), and having logs less than some finite length is to say that an entity has a precondition on how long they will train for (\ie for computational costs, which is common as one usually sets some finite number of epochs to train for as a hyperparameter). Lastly having only finitely many logs in $H_D$ is entailed by us restricting to only finite length logs and having $D$ finite (or only working with finitely many $m$-combinations of data points of $D$ in the case where $D$ is infinite, which is required for our following argument).

Thus by multi-dimensional Chebyshev's (and more specifically Markov's inequality):
\begin{multline}
    \mathbb{P}(|g(\mathbf{w},\mathbf{\hat{x}})-\mathbf{\mu}|_2 \geq \epsilon) \\ = \mathbb{P}(|g(\mathbf{w},\mathbf{\hat{x}})-\mathbf{\mu}|_2^2 \geq \epsilon^2) \\ \leq \mathbb{E}(|g(\mathbf{w},\mathbf{\hat{x}})-\mathbf{\mu}|_2^2) / \epsilon^2,
\end{multline}
where the first equality is true by the fact that $|g(\mathbf{w},\mathbf{\hat{x}})-\mathbf{\mu}|_2 \geq \epsilon$ iff $|g(\mathbf{w},\mathbf{\hat{x}})-\mathbf{\mu}|_2^2 \geq \epsilon^2$ (by monotonicity of squaring) and the last inequality is just markov's inequality. But as we squared, note  $\mathbb{E}(|g(\mathbf{w},\mathbf{\hat{x}})-\mathbf{\mu}|_2^2)$ is just the trace of the covariance of $g(\mathbf{w},\mathbf{\hat{x}})-\mathbf{\mu}$ which is $\sigma^2/b$, thus we have $\mathbb{P}(|g(\mathbf{w},\mathbf{\hat{x}})-\mathbf{\mu}|_2 \geq \epsilon) \leq \sigma^2 /(b \epsilon^2)$.

Furthermore if $\mathbf{\hat{x}} \in D$ and $\mathbf{\hat{x}}_1 \in D'$, then $|g(\mathbf{w},\mathbf{\hat{x}}) - g(\mathbf{w},\mathbf{\hat{x}}_1)|_2 \leq 2\epsilon$ with probability greater than $(1 - \sigma^2 / (b\epsilon^2))^2$ (i.e., the probability of them both being within an $\epsilon$ ball of $\mu$). This is as if both $|g(\mathbf{w},\mathbf{\hat{x}}) - \mathbf{\mu}| \leq \epsilon$ and $|g(\mathbf{w},\mathbf{\hat{x}}_1) - \mathbf{\mu}| \leq \epsilon$, then by triangle inequality $|g(\mathbf{w},\mathbf{\hat{x}}) - g(\mathbf{w},\mathbf{\hat{x}}_1)|_2 \leq |g(\mathbf{w},\mathbf{\hat{x}}) - \mathbf{\mu}| + |g(\mathbf{w},\mathbf{\hat{x}}_1) - \mathbf{\mu}| \leq 2\epsilon $. Moreover we have $\mathbb{P}(|g(\mathbf{\hat{x}}_0) - g(\mathbf{\hat{x}}_i)|_2 \leq 2\epsilon~for~some~\mathbf{\hat{x}}_i\in D') \geq (1 - \sigma^2 / (b\epsilon^2))(1-(\sigma^2 / (b\epsilon^2))^n)$ where now the second term represent the probability for at least one $\mathbf{\hat{x}}_i\in D'$ to produce $g(\mathbf{w},\mathbf{\hat{x}}_i)$ withing an $\epsilon$ ball of $\mathbf{\mu}$

Now note there are only finitely many parameters in $H_D(W)$ due to the restrictions on lengths of logs and the number of logs, thus let $(1 - \sigma_{max}^2 / (b\epsilon^2))(1-(\sigma_{max}^2 / (b\epsilon^2))^n)$ where $\sigma_{max}^2$ is the largest trace of the covariance matrix of $g(\mathbf{w},\mathbf{x}) - \mathbf{\mu}_{\mathbf{w}}$ over all $\mathbf{w} \in H_D(W)$ (and thus this term represent the minimum probability for any step to be within a $2\epsilon$ ball).

Thus the probability for us to be within a $2\epsilon$ over an entire log (which has at most length $m$) is greater than or equal to $((1 - \sigma_{max}^2 / (b\epsilon^2))(1-(\sigma_{max}^2 / (b\epsilon^2))^n))^m$, and furthermore for this to be true for all logs in $H_D$ (which there $\alpha$ of) is greater than or equal to $((1 - \sigma_{max}^2 / (b\epsilon^2))(1-(\sigma_{max}^2 / (b\epsilon^2))^n))^{m\alpha}$. 

So we have proved that with probability lower bounded by $((1 - \sigma_{max}^2 / (b\epsilon^2))(1-(\sigma_{max}^2 / (b\epsilon^2))^n))^{m\alpha}$ the dataset $D'$ forges $D$ with $d = \ell_2$, $2\epsilon$ and $g$. 

\begin{theorem}[Probablistic Forging]
\label{thm:prob_forging}
If $D$ and $D'$ are sampled i.i.d $\mathfrak{D}$ and consists of minibatches of size $b$ with the training update function $g$ being a mean sampler over the minibatch, and $|D'| = n$ and $|H_{D g,d,0}| = \alpha$ with every log $\{(\mathbf{w}_i,\mathbf{x}_i)\}_{i \in J} \in H_{D g,d,0}$ having $|\{(\mathbf{w}_i,\mathbf{x}_i)\}_{i \in J}| \leq m$. Then with probability at least $((1 - \sigma_{max}^2 / (b\epsilon^2))(1-(\sigma_{max}^2 / (b\epsilon^2))^n))^{m\alpha}$ $D'$ forges $D$ with $2\epsilon$ and $d = \ell_2$.
\end{theorem}

\begin{proof}
See above remark
\end{proof}

Thus we see some immediately important factors to forging with high probability: increasing $|D'| = n$, increasing the minibatch sizes $b$,  decreasing the maximum length of a log, or the number of logs (which would be possible by decreasing the size of $D$ and thus all combinations of its elements). In fact, we have the following corollary:

\begin{corollary}[The limit of Large Minibatch Sizes]
\label{cor:large_batches}
In the limit $b \rightarrow \infty$ for any $D$ and $D'$ both sampled i.i.d from $\mathfrak{D}$ consisting of minibatches of size $b$, with $|D| \geq n$ and $|D'| \geq n$ and $|H_{D g,d,0}|\leq \alpha$,$|H_{D' g,d,0}| \leq \alpha$, and maximum log sizes $m$, we have they are forgeable for every $\epsilon > 0$ for $g$ and $d$ as stated in Theorem~\ref{thm:prob_forging}.
\end{corollary}

\begin{proof}
Follows from the fact that if $D$ and $D'$ consists of minibatches of size $b$ with $|D'| \geq n$ (\ie contains at least $n$ minibatches) and $|H_{D g,d,0}| \leq \alpha$ with every log $\{(\mathbf{w}_i,\mathbf{x}_i)\}_{i \in J} \in H_{D g,d,0}$ having $|\{(\mathbf{w}_i,\mathbf{x}_i)\}_{i \in J}| \leq m$, then we know $D'$ forges $D$ with probability at least $((1 - \sigma_{max}^2 / (b\epsilon^2))(1-(\sigma_{max}^2 / (b\epsilon^2))^n))^{m\alpha}$ for any $\epsilon > 0$, and this probability goes to $1$ as $b \rightarrow \infty$. It is worth recalling that $\sigma_{max}$ was just the max variance of $g(\mathbf{w},\mathbf{x})$ over all $\mathbf{w} \in H_D(W)$ and hence independent of $b$ (\ie is a fixed constant). 

Completely analogously (as $D$ and $D'$ satisfy the same conditions) we have $D$ forges $D'$ for any $\epsilon > 0$ with probability going to $1$ as $b \rightarrow \infty$. 

Thus, the probability of both events occurring (the product of the individual probabilities) also goes to $1$ as $b \rightarrow \infty$ meaning we have forging maps in both directions, they are forgeable for all $\epsilon >0$ with probability going to $1$ as $b \rightarrow \infty$.
\end{proof}

\paragraph{Remark on when $b$ is up to the adversary.} This corollary is particularly important when the adversary (\ie the model owner) may choose their own minibatch size $b$, \ie they have a set of data points $D = \{\mathbf{x}_1,\cdots,\mathbf{x}_n\}$ and can modify $D$ to then consist of minibatches $\Tilde{D} = \{\mathbf{\hat{x}}_1,\cdots,\mathbf{\hat{x}}_l\}$ of size $b$. Then, by making $b$ as large as possible, they can increase the probability of finding a forging dataset $D'$ drawn i.i.d from $\mathfrak{D}$.

However, note $b$ needs to satisfy $b \leq n$ as the adversary only has $n$ data points. When the adversary can continue to sample from $\mathfrak{D}$ to add points, they could increase $n$ arbitrarily and thus $b$ could be as large as they like. Then, they can increase their probability of success arbitrarily so that there is some other dataset $D'$ that forges it for any $\epsilon > 0$. Furthermore, if any point $\mathbf{x} \in \mathfrak{D}$ has probability $0$ of being drawn, then $D'$ has probability $1$ of not containing any of the points in $D$. 

This implies that with probability $1$, a particular dataset (and in particular any of its points) can be forged with another disjoint dataset if the adversary has freedom over $b$ and can sample from $\mathfrak{D}$ indefinitely. In particular, in the limit we have probability $1$, \ie there is a finite case after obtaining minibatch size $b > \beta_D$ (where $\beta_D$ is some finite number dependent on $D$) by adding points that have a non-zero probability of being forged for almost all $D$. Thus, by existence by non-zero probability, there is at least one larger finite dataset $\Tilde{D} \subset D$ which can be forged (as $\emptyset$ has $0$ measure always, thus any non-zero measure set can not be empty).

\begin{theorem}[Existence of Forging Adversary]
\label{thm:forging_exists}
If an adversary has freedom over $b$ and access to $\mathfrak{D}$ such that they can sample from $\mathfrak{D}$ (where any point in $\mathfrak{D}$ has $0$ probability) indefinitely, then for every $\epsilon > 0$, for almost every dataset $D$ there exists $\Tilde{D} \supset D$ which can be forged by a disjoint dataset $D'$ (and in particular $\Tilde{D}$ and $D'$ are finite).
\end{theorem}

\begin{proof}
Simply the above remark, noting probability $1$ means "almost every" (i.e., every dataset minus a measure $0$ set)
\end{proof}

\subsection{Consequences of Theorem~\ref{thm:forging_exists} for Unlearning of minibatch SGD}
\label{ssec:conseq_of_thm_3}

We now explain the practical consequence of Theorem~\ref{thm:forging_exists} to minibatch SGD, which is the de facto training algorithm for deep neural networks. In particular, we have:

\begin{corollary}
\label{cor:unl_mini_batch_SGD}
Unlearning on minibatch SGD is not well-defined over the logs (and thus the models) of finite datasets obtained from a distribution $\mathfrak{D}$, for any such distribution $\mathfrak{D}$.
\end{corollary}

\begin{proof}
To prove that unlearning is not well-defined when it only involves logs (or models), we need simply that for any verification threshold $\epsilon > 0$ there $\exists D \subset \mathfrak{D}$ such that $D$ can be forged by a disjoint dataset $D'$; thus the logs (no matter the threshold) carry no information of what dataset was used to train the model, and thus by the remark of \S~\ref{ssec:conseq_to_defining_unl} unlearning is not well-defined for $D$. This, in turn, implies that in general unlearning is not well-defined over finite datasets (as being not well-defined for one instance in the domain, in this case of finite datasets, is sufficient). Note, we assume the verifier can not use $\epsilon = 0$ simply because of floating point precision. 

By Theorem~\ref{thm:forging_exists} we know for every $\epsilon$, almost all finite datasets $D$ that could be obtained i.i.d are contained in a larger dataset $\Tilde{D}$ which can be forged with $\epsilon$ by a disjoint dataset $D'$. Note here that $D$ does not need to have been obtained by i.i.d sampling, just the fact $D$ \textit{could be} obtained by i.i.d sampling. Thus we have existence of $\Tilde{D}$, which completes the proof.
\end{proof}

This is the main conclusion of this paper: unlearning cannot be well-defined for minibatch SGD over the logs. This is best understood because forgeability implies that models or checkpoints (or distribution of models or checkpoints) are both unlearnt and not unlearnt at the same time. That is, the training log could be obtained by both not training with a particular data point, and by training with it and then employing a forging map.

Hence for the community to move forward with using properties of the training/unlearning algorithm outputs to define its unlearnability (as was used in all prior work), one needs to put constraints on the training process. This is necessary to remove the possibility of these pathological examples\footnote{We have by no means exhaustively characterized all such examples, we have just shown examples exist. There could be other aspects of training (besides using mean-sampler $g$) which can create this possibility. In fact, notably, we have not even shown what the examples we know exist look like. }. Then, auditing processes will need to verify that this process was indeed used by the model owner. Overall, this implies that unlearning should be defined \textit{algorithmically} and designed such that it is amenable to external scrutiny. 

One possible future direction in this vein is to design verifiable ML processes to give guarantees about the points used (or not used) in a particular computation. This would prove that an entity was not able to employ a forging map in lieu of unlearning. We discuss this further in \S~\ref{ssec:verifiable_ML}.

\section{An Instantiation of a Forging Attack}
\label{sec:attack}

Theorem~\ref{thm:forging_exists}, and subsequently Corollary~\ref{cor:unl_mini_batch_SGD}, analytically demonstrate the existence of forging examples. We now experimentally show the existence of forging maps on a practical dataset. This allows us to construct forging maps to study exact unlearning, and in particular PoUL for verifying exact unlearning through retraining. As mentioned earlier in \S~\ref{ssec:definition}, the attack we propose here does not necessarily bypass all PoL verification schemes, and hence one might still be content with using a verification scheme that defends against our attack to verify unlearning in practice. However note this essentially starts an arms race between better forging and better PoL verification without other constraints to rule out forging (see \S~\ref{ssec:verifiable_ML}).

The methodology for the attack we instantiate draws from the data ordering attack proposed by~\citeauthor{shumailov2021manipulating}~\cite{shumailov2021manipulating}. Their attack takes advantage of the stochasticity present in minibatches and the effect of large minibatch sizes. This is well-aligned with our analytical result in Theorem~\ref{thm:prob_forging} showing that large minibatch sizes facilitate forging. Specifically, we focus on a forging map between $D$ to $D' = D \setminus \mathbf{x}^*$ for some data point $\mathbf{x}^*$. This corresponds to a single point being unlearned. Moreover, we focus our empirical investigation on updates that originally used $\mathbf{x}^*$ to show $D$ can be forged.

The main takeaway from our experiments is simply that one can reach verification error $\epsilon$ (in a PoUL) below $10^{-4}$. For suitable large minibatches, we can reach $\epsilon \approx 10^{-6}$ (before adding the noise introduced by hardware). This we believe shows the practicality of forging, and motivates future work to improve algorithms for finding a forging map.

\subsection{Threat Model}

We now describe the specific threat model for this attack, which is the one of \S~\ref{sec:strawman}: a PoL where we have access to both the model and dataset used, and want to now replace all the steps in the PoL with a data point $\mathbf{x}^*$ with another as to now have a forging map to $D' = D\setminus \mathbf{x}^*$ for that particular PoL. 

We in particular focus on a random sample of steps with $\mathbf{x}^*$ at different checkpoints (and also varying what point $\mathbf{x}^*$ is) to illustrate the general effectiveness of this attack. Here by checkpoints, we mean the state of the model parameters saved after each epoch of training. 

The model used is a LeNet5~\cite{lecun1989backpropagation} trained on MNIST~\cite{10027939599}: our results in this section are meant to be understood as initial results towards future study. In all the experiments we work with a fixed learning rate $\eta = 0.01$ and the SGD optimizer.

\subsection{Attack Description}

We take a model trained for $N$ epochs before having a training update with $\mathbf{x}^*$ (in this case a minibatch containing the specific data point). This step takes initial parameters $\mathbf{w}_i$ to the parameters $\mathbf{w}_{i+1}$. Our first forging algorithm then is to:
\vspace{-1mm}
\begin{enumerate}
\itemsep0em
    \item Sample $n$ data points uniformly from $D/\mathbf{x}^*$
    \item Sample $M$ minibatches uniformly from the selected $n$ data points
    \item Select minibatch $\hat{\mathbf{x}}_j$ which minimizes $||\mathbf{w}_{i+1} - g(\mathbf{w}_i,\hat{\mathbf{x}}_j)||_{2}$, \ie sort the minibatches by the error $||\mathbf{w}_{i+1} - g(\mathbf{w}_i,\hat{\mathbf{x}}_j)||_{2}$ and select the minibatch with the least error as our "forging minibatch"
\end{enumerate}

Note that, ignoring sorting overhead, the cost of running this algorithm is $M$ gradient steps (as we compute a gradient for each of the $M$ randomly chosen minibatches in step 3). In general if one wants to forget $\mu$ datapoints, and trained $N$ epochs, then they need to forge $\mu N$ steps. Using our algorithm we then have a cost of $\mu N M$ gradient steps.

We further investigate factors which may increase or decrease our verification error, in particular: minibatch size, and number of data points sampled $n$. We also test if further greedy selection by interchanging individual data points in the selected best minibatch (rather than interchanging minibatches of data points) can further decrease error. Note, if we take $T-1$ greedy updates steps after running our original algorithm, each also looking at $M$ possible interchanging data points, our costs are now $\mu N MT$.

After demonstrating the above, we attempt to forge with an even smaller subset of $D$, in which case our forging algorithm is:
\vspace{-1mm}
\begin{enumerate}
\itemsep0em
    \item Sample size $n$ subsets from $D$
    \item Repeat algorithm 1 for the subset
\end{enumerate}

Being able to forge with these smaller size $n$ subsets shows we could forge $|D| -n$ points with only the $n$ points, \ie could forge having unlearnt large portions of the dataset.
\subsection{Results} 

\paragraph{Forging when varying $n$: Larger is marginally better.}

In Figure~\ref{fig:err_vs_n_samples} we plot our verification error for 100 different trial runs (each of which had randomized checkpoints and $\mathbf{x}^*$) as a function of $n$, with fixed minibatch size of $1000$. Observe that there is a modest decrease with increased $n$, however, we always had error $\approx 10^{-6}$. This suggests that the number of samples one uses to then create minibatches is not of significant importance. In particular, for more efficient forging algorithms one can get away with only sampling a smaller set to then construct minibatches from.

\paragraph{Minibatch Size: Larger is significantly better.}
In Figure~\ref{fig:l2_difs_batch_size}, we plot the PoL's verification error as a function of the minibatch size. In support of our theory (\S~\ref{ssec:Mean_Sampling}) which stated that in the limit of large minibatch sizes we can always forge, we see increasing the minibatch size plays a significant role in reducing the verification error, dropping several magnitudes. For large minibatch sizes, it reaches $\approx 10^{-5}$ error. In fact, the trend notably looks like what one might expect from a $1/b$ factor in the minibatch size, matching the formulation  (modulo exponents) stated in Theorem~\ref{thm:prob_forging}.

\paragraph{Greedy selection of individual points: improves slightly.}
We next employ the greedy algorithm where we update one data point at a time (rather than substituting for the entire minibatch). We see in Figure~\ref{fig:l2_difs_greedy} that we can slightly improve the verification error and that the improvement tends to converge after less than $10$ updates. Note that we plotted the relative improvement, \ie error $1$ is the same as before switching to the greedy updates and $0.9$ is $0.9$ times the error we had before the greedy updates. The main takeaway from this is that a greedy procedure does not significantly improve error, especially when compared to the effect of minibatch sizes.

\paragraph{Smaller sets: they work.}
Now when considering fixed smaller sets of 1000 to sample from, depicted in Figure~\ref{fig:err_disjoint_set}, we see we still have error less than $10^{-4}$. This is larger than $~\approx 10^{-6}$ we were achieving before, but we are also working with $1/60$ the fraction of the size of the larger datasets from before, indicating that smaller datasets can still forge reasonably well. This notably suggests that in general, in the limit of large minibatch sizes, a smaller dataset is equivalent to a larger dataset, though we leave that direction for future work.

\paragraph{Main Takeaways.} The results we presented show that forging for $\epsilon$ several magnitudes less than $1$ is practical. Thus, forging is not only possible for very low thresholds, but is also relatively easy as our algorithm simply samples randomly. Future work may extend our initial results by investigating how these results hold in different domains, (\ie for language models). Alternatively, if given the hardware capacity, one might look into how much lower one can decrease the error by increasing the minibatch sizes past what we tested. Furthermore, our results for smaller sets indicate that they can still produce low error; an interesting follow-up question is how low can the error for smaller datasets be pushed. Related to that, can one find a particular smaller dataset that produces the same error as we were getting when forging with the larger datasets? This question seems similar to dataset distillation, however, we leave formal or empirical relations to future work.

\begin{figure}[t]
    \vspace{-5mm}
    \centering
    \includegraphics[width=\linewidth]{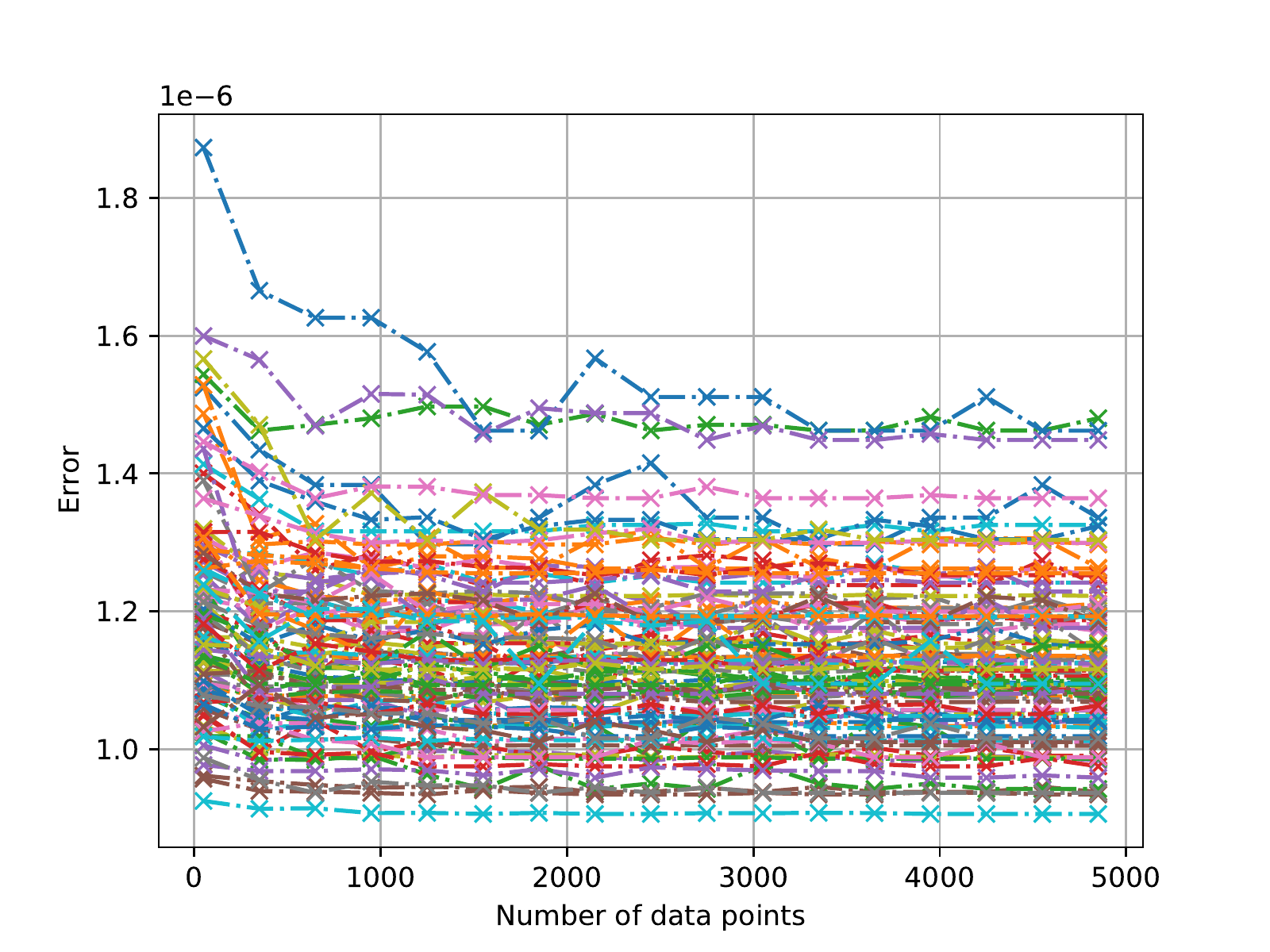}
    \caption{Verification error as a function of the number of samples in the dataset: one can see the error is always in the order of $10^{-6}$ and there is a slight decrease in the error as the dataset becomes larger. Note here we used a fixed minibatch size of $1000$, and the colors denote different runs.}
    \label{fig:err_vs_n_samples}
\end{figure}

\begin{figure}[t]
\vspace{-6mm}
    \centering
    \includegraphics[width=\linewidth]{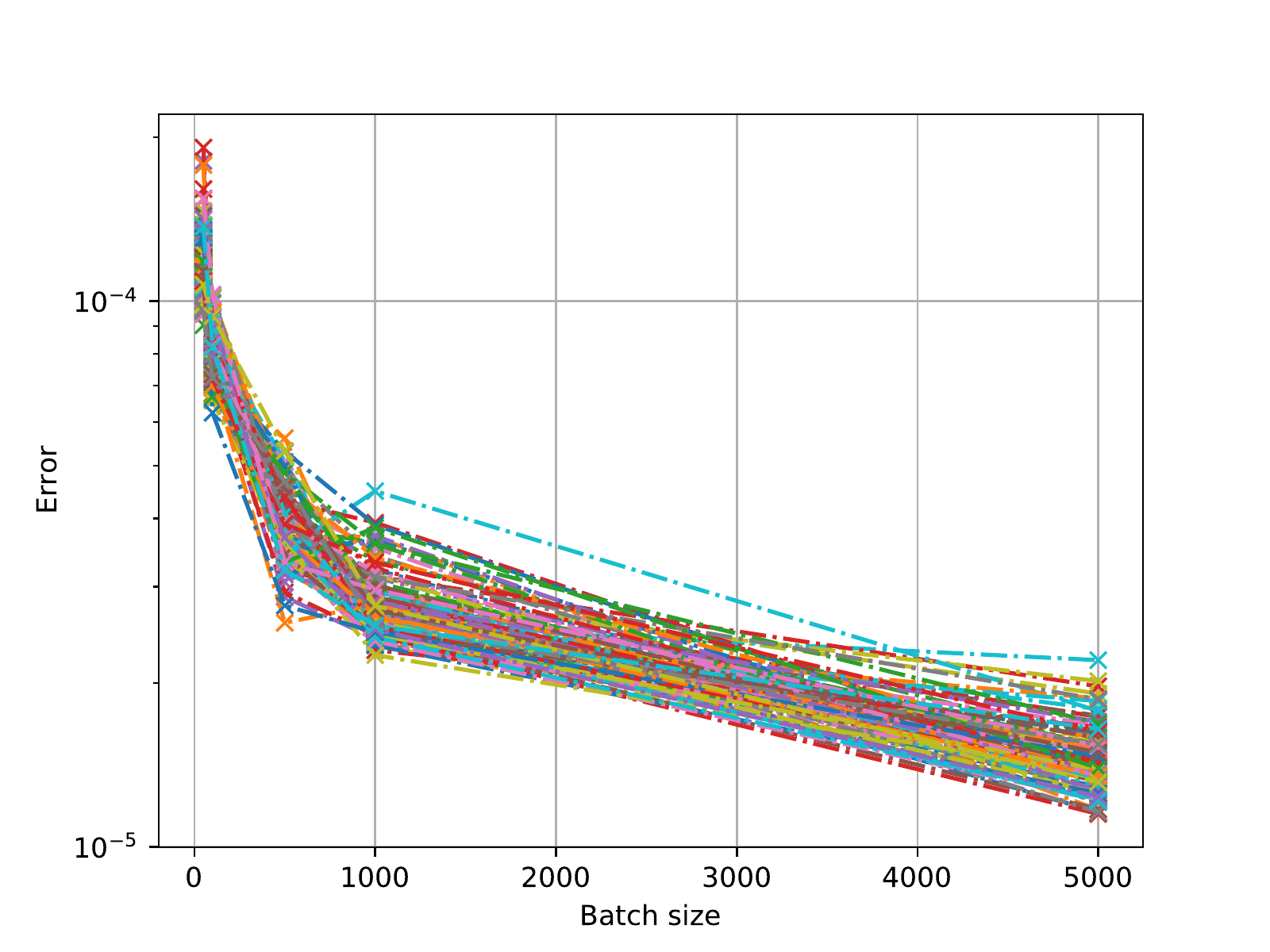}
    \caption{Verification error as a function of the minibatch size: the error drops drastically as the minibatch size increases from $0$ to $1000$. Afterwards the drop becomes more gradual. Note again the colors denote different runs.}
    \label{fig:l2_difs_batch_size}
\end{figure}

\begin{figure}[t]
\vspace{-6mm}
    \centering
    \includegraphics[width=\linewidth]{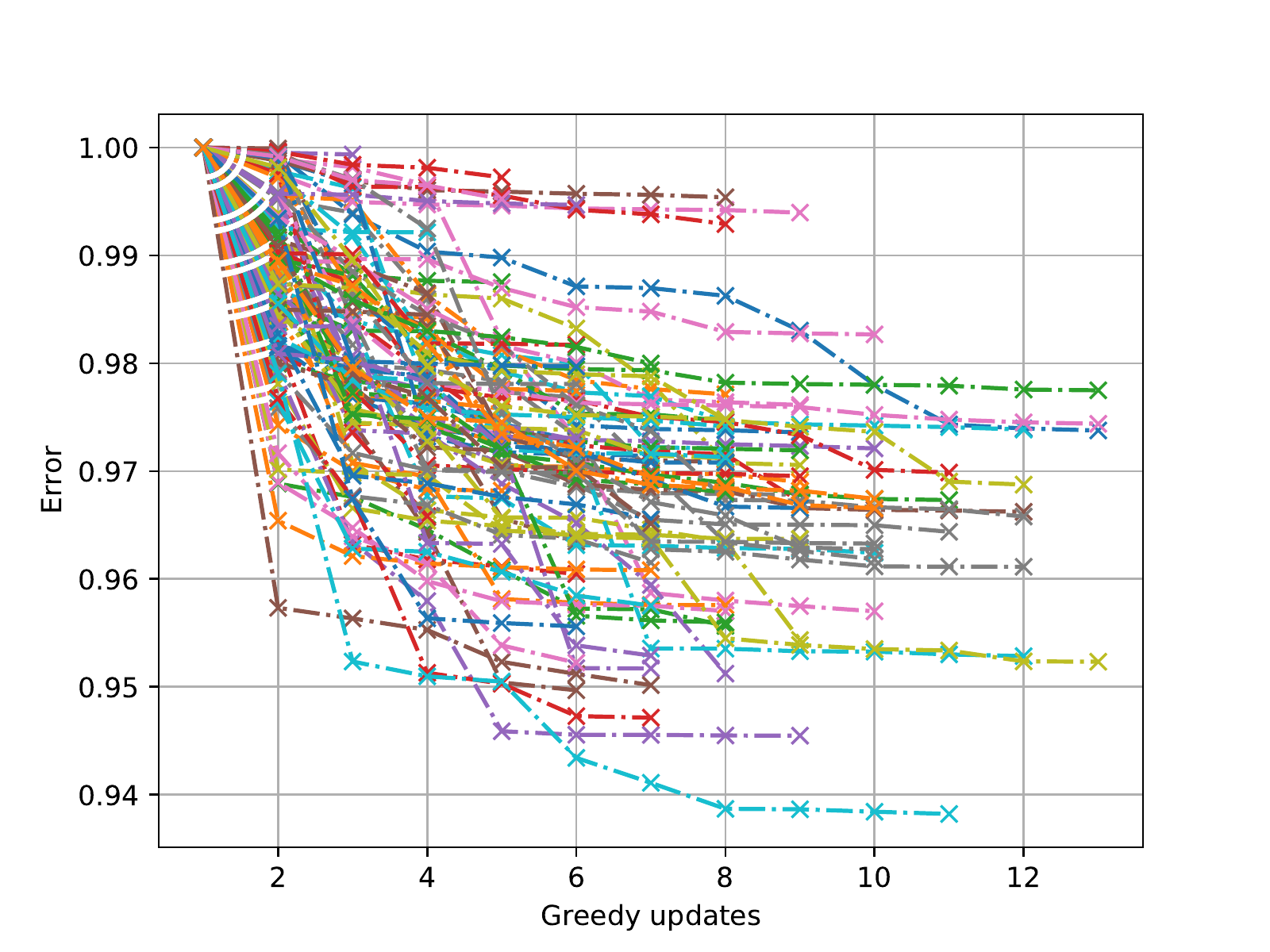}
    \caption{Relative verification error as a function of the number of greedy updates (selecting individual points rather than substituting for entire minibatches). The curves are generated by repeating the experiment on 100 individual models, starting from a randomly found minimum. Note how we only observe a marginal benefit by increasing the number of updates. Furthermore, not again that different colors denote different runs.\vspace{-5mm}}
    \label{fig:l2_difs_greedy}
\end{figure}

\begin{figure}[t]
\vspace{-6mm}
    \centering
    \includegraphics[width=\linewidth]{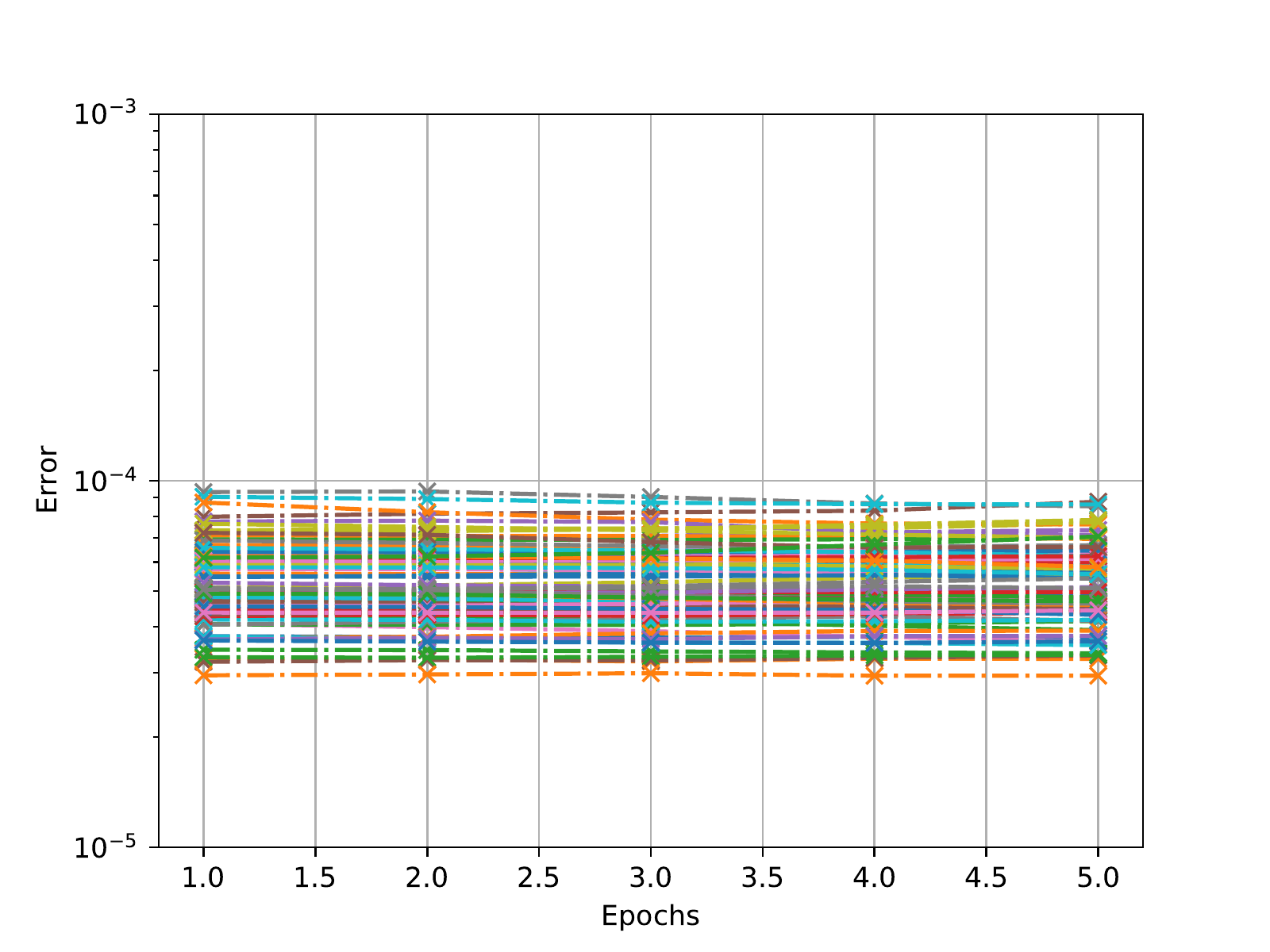}
    \caption{Verification error when forging using samples from a smaller dataset, plotted against the epoch where we are forging. One can see the error is still quite low (less than $10^{-4}$). Note again the different colors denote different runs.\vspace{-5mm}}
    \label{fig:err_disjoint_set}
\end{figure}

\section{Discussion}

\subsection{Relation to Information Leakage}
Our work shows that there are cases where an entity has plausible deniability that they did (not) train with a particular point, and that this likelihood increases with batch size. This is seemingly a strong privacy-like statement. Because our proof of Theorem \ref{thm:prob_forging} relies on a form of convergence to the mean which increasing batch size induces, forging is in this case akin to having a very average training run. One might interpret this as meaning that this run leaks less information about the specific data points used.

In fact, one might generally view forging by datasets drawn from the same distribution (or larger finite set) as representing how stable training is to the underlying distribution and how little it leaks about specific data points. Investigating the privacy aspects presented by this could be an interesting line of future work. Notably, this is quite different from the flavor of guarantee provided by differential privacy \cite{dwork2006calibrating} as we are not adding noise, but quantifying the underlying noise of the distribution and how the training algorithm "averages it out".

\subsection{On the Need for Verifiable Training and Unlearning Algorithms}
\label{ssec:verifiable_ML}

Before getting to why verifiable computing would be beneficial, it is worth noting why further constraining the unlearning process would not work. Simply not allowing a log to ever be changed would not work as an entity can simply create a new log which forged the previous log. Perhaps another constraint of only ever being able to train once would circumvent this issue, but it is difficult to enforce. 
Essentially it seems difficult (though we will not claim impossible) to put enough constraints on learning so that the information of the valid log can never lead to a forged log.

Thus, adding verifiable computing to the paradigm seems like a promising direction. In this case one knows exactly what data point was used for a specific computation; however, this still does not prevent an entity from using that information to then compute with the forged example. Rather there would need to be another primitive which controls the information flow allowed for the computations. For example, putting restrictions on where the data was obtained from in an effort to prevent forging datasets. Future work will have to investigate the data provenance requirements needed for this given that our existence of forging examples \textit{does not tell us what exactly they look like}.

Nevertheless, this is a direction we expect future work will address, and we hope that the forging framework we developed might prove useful in also proving when it might not exist given further conditions.

\subsection{Extending the Forging Framework}
\label{ssec:extending_framework}

Building on that last point, we expect that the forging framework can itself be further studied as its own subject, or variations of it as we described in \S~\ref{sec:forging}. This is to say that though we set out specifically to show how unlearning with minibatch SGD is not well-defined, the tools we developed could be used in quite a more general fashion for studying ML.

We already illustrated an interesting property of our current framework in Theorem~\ref{thm:charaterizing_equiv_class}, which is the $\epsilon =0$ equivalence class,
and possible variations of the forging framework may admit other useful properties. In particular, it might be possible to use the $\epsilon$ we get from forgeability (or variations of it) to define a distance metric over datasets; we were not able to obtain a proof that the metric $\Bar{d}(D,D') = inf \{\epsilon:~D~\text{and}~D'~\text{are}~\epsilon~\text{forgeable}\}$ defined over all datasets (modulo the $\epsilon = 0$ forgeable equivalence class) satisfies the triangle inequality, though we also do not have particular counter-example for it not to satisfy it (the specific issue is trying to compose forging maps, as note the error in the input and the error in the output are different in our particular flavour of forging). In particular, a metric (in the formal mathematical sense) would give a topology over the datasets, and would open further questions about what properties the topology satisfies: for example can such a space be complete under some constraints on $g$?

Another particularly interesting question is how does the space of forgeable datasets, perhaps in particular the $\epsilon = 0$ equivalence classes, change with the model architecture (note throughout this paper we assumed the model architecture was fixed). Understanding this might give insights into the variance between model parameters obtained for certain datasets (modulo the $\epsilon = 0$ equivalence class). Future work will analyze the benefits and drawbacks of such variance.

\subsection{Importance to ML Governance}
\label{ssec:governance}

The concept of \textit{machine learning governance}~\cite{chandrasekaran2021sok} aims to standardize the interactions between principals in the entire life cycle (\eg deploy, use, patch, and retire) of ML models. The goal is to ensure that alongside all the benefits provided by ML models, the risks they create are controlled and restricted. 
Two key aspects of ML governance are \textit{accountability} and \textit{assurance}, where the former means to hold some entities accountable for ML failures and the latter assures that an ML model meets certain security requirements. Here, we answer the following questions: (1) why are current auditing techniques not sufficient to hold model owners accountable for unlearning in the presence of forging, (2) why one cannot attribute (\ie another way to achieve accountability) a failure or property of a model purely to data when there exists forging examples, and (3) after realizing the existence of forging, how does a model owner provide assurance to users that the latter's data is indeed unlearned?

\paragraph{Auditing Unlearning} Machine learning auditing is defined as the process of formal examination/verification of claimed properties of ML models, done by a regulator. When it comes to exact unlearning, the regulator wants to verify that the model owner has indeed removed the impact of the data to be unlearned from the model. It is worth noting that \citeauthor{chandrasekaran2021sok}~\cite{chandrasekaran2021sok} suggests ML auditing can be supported by logging (\eg PoL). However, as described in \S~\ref{ssec:conseq_to_defining_unl}, when there exists a forging map between the dataset and itself with the unlearned data points removed (or another dataset without these points), then it is not possible to audit unlearning by a PoL log. Instead, the regulator may need auditing techniques at the level of the algorithms, such as verifiable training mentioned in \S~\ref{ssec:verifiable_ML}.

\paragraph{Data Attribution} Data attribution is the process of associating a certain model behaviour to a particular data point (or a subset of data points). This is also sometimes known as influence~\cite{koh2017understanding}. What our work on forging shows is that for a specific training log, the minibatch that makes the update from a checkpoint to the next in the log is not unique. More generally, this means that attributing failures (or properties) of a model in the parameter space to a specific set of data points is ill-defined when forging exists. That is, analogous to how one must define unlearning as being a product of using a specific algorithm, one cannot attribute any failures of a model to the presence/absence of certain training points without stronger assumptions on the training process: for example that the process does not allow for any forging. %

For instance, there exists a naive reasoning about fairness of ML that bias in data is the only source of unfairness. Forging is a counterexample for this, because forging shows that we cannot attribute unfairness to data without ruling out specific forging examples through algorithmic constraints on top of having trained with some update rule $g$. This is also supported by the emerging belief that unfairness stems from both issues with data and the algorithmic choices (loss, architecture, etc.) \cite{mitchell_algorithmic_2021}, and emphasizes that verifying unfairness requires auditing the fairness of the algorithm as well.

This ill-definedness is also true beyond failures of the model. 
For example, a good performance now cannot be exclusively attributable to a specific dataset unless forging examples are ruled out for that log. %
However, this can in itself be informing. Recall that we mentioned in \S~\ref{ssec:extending_framework} the forging framework may provide a useful tool for studying ML, and $\epsilon = 0$ classes in some sense characterize all datasets that lead to similar performance. 
For instance, understanding and studying the $\epsilon = 0$ class for a dataset producing high-performance models could help tasks like dataset distillation by finding smaller datasets that can lead to similar-performance models under suitable restrictions on the data sequence sampling.

\paragraph{Assurance} Assurance is what the model owners wish and need to provide to the users to gain their trust: that the models are guaranteed to meet certain security and/or privacy standards, and the risks in contributing training data for these models are controlled to certain degrees. In the context of unlearning, the benign model owners would like to assure users (or data providers) that unlearning is done correctly by providing them with a verifiable certificate. However, note that even if a method to verify exact unlearning is introduced (\eg verifiable training-based methods described in \S~\ref{ssec:verifiable_ML}), it would likely be too costly for a normal model user to verify the training of ML models. Hence, achieving assurance will require future work to design efficient solutions. Alternatively, the cost of verification could be offset to a regulator trusted by the users to act on their behalf. %

\section{Conclusions}

Does our work jeopardize all progress made thus far towards unlearning? Not exactly. It means we now need to go back and redefine what unlearning is, such as by adding more restrictions on the training process to make it harder to find forging examples (\S~\ref{ssec:verifiable_ML}). We also need to see whether training without the data point is a well-defined way to define unlearning (in a way that is no longer susceptible to forging examples). That said, our work does show current faults in defining unlearning, and the more general problem of data attribution when analyzing both correct and incorrect ML predictions. We also outlined future avenues for the concept of forgeability to help understand ML algorithms better.

\section*{Acknowledgments}

We would like to acknowledge our sponsors, who support our research with financial and in-kind contributions: CIFAR through the Canada CIFAR AI Chair, DARPA through the GARD project, Intel, Meta, NFRF through an Exploration grant, and NSERC through the COHESA Strategic Alliance. Resources used in preparing this research were provided, in part, by the Province of Ontario, the Government of Canada through CIFAR, and companies sponsoring the Vector Institute. We would like to thank members of the CleverHans Lab for their feedback. We also thank the reviewers for their useful feedback, and Ting Wang for being our shepherd.

\bibliographystyle{plainnat}
\bibliography{references.bib}

\appendix
\section{Table of Notations}
\label{sec:notations}

\begin{table}[h!]
    \centering
    \begin{tabular}{ll}
        \toprule
        \textbf{Symbol} &  \textbf{Explanation}\\
        \midrule
            $D$, $D'$, $\Tilde{D}$ & datasets \\
            $\mathbf{x}_i$ & $i^{th}$ input in a dataset\\
            $y_i$ & label of $x_i$\\
            $\mathfrak{X}$ & data space \\
            $c$ & number of classes \\
            $\mathbf{w}$ & model parameters\\
            $W$ & parameter space\\
            $N$ & dimensionality of the parameter space\\
            $\mathbf{w_0}$ & model parameters after initialization\\
            $\mathbf{w_t}$ & model parameters after $t$ updates\\
            $M_\mathbf{w}$ & supervised machine learning model \\
            & with parameters $w$\\
            $\mathcal{L}$ & loss function\\
            $g$ & update rule of model, \eg SGD\\
            $\mathbf{\hat{x}}$ & minibatch of inputs\\
            $b$ & minibatch size\\
            $\eta$ & learning rate\\
            PoL & proof-of-learning\\
            PoUL & proof-of-unlearning\\
            $d$ & distance metric in the parameter space,\\
            & \eg $\ell_p$ or cosine distance\\
            $\epsilon$ & threshold to determine if a PoL is valid\\
            $H_{D g,d,\epsilon}$ or $H_{D}$ & set of all valid $(g,d,\epsilon)$ logs stemming\\
            & from $\mathbf{w}_0$ using dataset $D$\\
            $H_D(w)$ & set of logs with $\epsilon=0$\\ 
            $B$ & forging map (in the space of set of logs)\\
            $B_\epsilon$ & forging map where output $H$ has \\
            & error of $\epsilon$\\
            $\mathbf{x}^*$ & data points required to be unlearned\\
            $D / \mathbf{x}^*$ & dataset $\mathbf{x}^*$ removed\\
            $L$ & lipschitz constant \\
            $Ball_{\epsilon}(\mathbf{w})$ & distance ball in the parameter space,\\
            &  centering at $\mathbf{w}$ with radius $\epsilon$\\
            $\mathfrak{D}$ & distribution of datasets\\
            $\mu$ & mean of random variable $g(\mathbf{w},\mathbf{x})$\\
            $\sigma^2$ & trace of covariance matrix of $g(\mathbf{w},\mathbf{x})$\\
            $\sigma_i^2$ & variance of the ith component of $g(\mathbf{w},\mathbf{x})$\\
            $\mathbb{P}$ & probability\\
            $n$ & number of minibatches in a dataset\\
            $\alpha$ & number of PoL logs in $H_D$\\
            $m$ & length of the longest PoL log in $H_D$\\
            $\beta_D$ & the minibatch size needed for dataset \\
            & $D$ to be forged\\
        \bottomrule
    \end{tabular}
    \caption{Notations}
    \label{tab:notations}
\end{table}

\end{document}